\pdfoutput=1

\documentclass[a4paper,envcountsame]{llncs}

\usepackage[english]{babel}
\usepackage[utf8]{inputenc}

\usepackage[unicode]{hyperref}
\usepackage{enumerate}

\usepackage{paralist}  
\usepackage{longtable} 
\usepackage{multicol}  

\usepackage{aixi}

\usepackage[labelfont=bf]{caption}

\def\Km{{\mathit{Km}}}                  
\def\BR{{B\!R}}                         
\def\nBR{{\overline{B}\!R}}             
\def\BnR{{B\!\overline{R}}}             
\def\nBnR{{\overline{B}\!\overline{R}}} 

\usepackage{listings}
\lstdefinelanguage{pseudo}{
	morekeywords={program, if, then, else, while, for, do, not, return, output, true, false},
	sensitive=false,
	morecomment=[l]{//},
	morecomment=[s]{/*}{*/},
	morestring=[b]",
}
\AtBeginDocument{

}

\title{Solomonoff Induction Violates Nicod's Criterion\thanks{The final publication is available at \url{http://link.springer.com/}.}}
\author{Jan Leike \and Marcus Hutter}

\titlerunning{Solomonoff Induction Violates Nicod's Criterion}
\authorrunning{J.\ Leike and M.\ Hutter}
\institute{
  Australian National University \\
  \texttt{\{jan.leike|marcus.hutter\}@anu.edu.au}
}

\begin{document}

\maketitle

\begin{abstract}
\emph{Nicod's criterion} states that
observing a black raven is evidence for the hypothesis $H$
that all ravens are black.
We show that Solomonoff induction does not satisfy Nicod's criterion:
there are time steps in which
observing black ravens \emph{decreases} the belief in $H$.
Moreover, while observing any computable infinite string compatible with $H$,
the belief in $H$ decreases infinitely often when using
the unnormalized Solomonoff prior,
but only finitely often when using
the normalized Solomonoff prior.
We argue that
the fault is not with Solomonoff induction;
instead we should reject Nicod's criterion.
\end{abstract}


\section{Introduction}
\label{sec:introduction}

Inductive inference,
how to generalize from examples,
is the cornerstone of scientific investigation.
But we cannot justify the use of induction on the grounds that
it has reliably worked before,
because this argument presupposes induction.
Instead, we need to give \emph{deductive} (logical) arguments for the use of induction.
Today we know a formal solution to the problem of induction:
Solomonoff's theory of learning~\cite{Solomonoff:1964,Solomonoff:1978},
also known as \emph{universal induction} or \emph{Solomonoff induction}.
It is a method of induction based on
Bayesian inference~\cite{Jaynes:2003} and algorithmic probability~\cite{LV:2008}.
Because it is solidly founded in abstract mathematics,
it can be justified purely deductively.

Solomonoff defines a prior probability distribution $M$
that assigns to a string $x$ the probability that
a universal monotone Turing machine prints something starting with $x$
when fed with fair coin flips.
Solomonoff's prior encompasses \emph{Ockham's razor}
by favoring simple explanations over complex ones:
algorithmically simple strings have short programs and
are thus assigned higher probability than complex strings that
do not have short programs.
Moreover, Solomonoff's prior respects
\emph{Epicurus' principle} of multiple explanation
by never discarding possible explanations:
any possible program that explains the string contributes
to the probability~\cite{Hutter:2007universal}.

For data drawn from a computable probability distribution $\mu$,
Solomonoff induction will converge to
the correct belief about any hypothesis~\cite{BD:1962}.
Moreover, this can be used to produce reliable predictions extremely fast:
Solomonoff induction will make a total of at most $E + O(\sqrt{E})$ errors
when predicting the next data points,
where $E$ is the number of errors of
the informed predictor that knows $\mu$~\cite{Hutter:2001error}.
In this sense, Solomonoff induction solves the induction problem~\cite{RH:2011}.
It is incomputable,
hence it can only serve as an ideal
that any practical learning algorithm should strive to approximate.

But does Solomonoff induction live up to this ideal?
Suppose we entertain the hypothesis $H$ that all ravens are black.
Since this is a universally quantified statement,
it is refuted by observing one counterexample: a non-black raven.
But at any time step, we have observed only a finite number of
the potentially infinite number of possible cases.
Nevertheless, Solomonoff induction
maximally confirms the hypothesis $H$ asymptotically.

This paper is motivated by
a problem of inductive inference extensively discussed in the literature:
the \emph{paradox of confirmation},
also known as \emph{Hempel's paradox}~\cite{Hempel:1945}.
It relies on the following three principles.
\begin{itemize}
\item \emph{Nicod's criterion}~\cite[p.\ 67]{Nicod:1961}:
	observing an $F$ that is a $G$
	increases our belief in the hypothesis that all $F$s are $G$s.
\item \emph{The equivalence condition}:
	logically equivalent hypothesis are confirmed or disconfirmed
	by the same evidence.
\item \emph{The paradoxical conclusion}:
	a green apple confirms $H$.
\end{itemize}
The argument goes as follows.
The hypothesis $H$ is logically equivalent to
the hypothesis $H'$ that all non-black objects are non-ravens.
According to Nicod's criterion,
any non-black non-raven, such as a green apple, confirms $H'$.
But then the equivalence condition
entails the paradoxical conclusion.


The paradox of confirmation has been discussed extensively
in the literature on the philosophy of science%
~\cite{Hempel:1945,Good:1960,Mackie:1963,Good:1967,Hempel:1967,Maher:1999,Vranas:2004};
see \cite{Swinburne:1971} for a survey.
Support for Nicod's criterion is not uncommon%
~\cite{Mackie:1963,Hempel:1967,Maher:1999} and no consensus is in sight.

Using results from algorithmic information theory
we show that Solomonoff induction avoids the paradoxical conclusion
because it does not fulfill Nicod's criterion.
There are time steps when (counterfactually) observing a black raven
disconfirms the hypothesis that all ravens are black
(\autoref{thm:M-decreases} and \autoref{cor:Mnorm-decreases}).
In the deterministic setting
Nicod's criterion is even violated infinitely often
(\autoref{thm:decrease-io} and \autoref{cor:Mnorm-decreases-io}).
However, if we \emph{normalize} Solomonoff's prior
and observe a deterministic computable infinite string,
Nicod's criterion is violated at most finitely many times
(\autoref{thm:Mnorm-on-sequence}).
Our results are independent of the choice of the universal Turing machine.
A list of notation can be found on
\hyperref[app:notation]{page~\pageref*{app:notation}}.

\section{Preliminaries}
\label{sec:preliminaries}

Let $\X$ be some finite set called \emph{alphabet}.
The set $\X^* := \bigcup_{n=0}^\infty \X^n$ is
the set of all finite strings over the alphabet $\X$, and
the set $\X^\infty$ is
the set of all infinite strings over the alphabet $\X$.
The empty string is denoted by $\epsilon$, not to be confused
with the small positive rational number $\varepsilon$.
Given a string $x \in \X^*$, we denote its length by $|x|$.
For a (finite or infinite) string $x$ of length $\geq k$,
we denote with $x_{1:k}$ the first $k$ characters of $x$,
and with $x_{<k}$ the first $k - 1$ characters of $x$.
The notation $x_{1:\infty}$ stresses that $x$ is an infinite string.
We write $x \sqsubseteq y$ iff $x$ is a prefix of $y$, i.e.,
$x = y_{1:|x|}$.

A \emph{semimeasure} over the alphabet $\X$ is
a probability measure on the probability space
$\X^\sharp := \X^* \cup X^\infty$
whose $\sigma$-algebra is generated by the \emph{cylinder sets}
$\Gamma_x := \{ xz \mid z \in \X^\sharp \}$
~\cite[Ch.\ 4.2]{LV:2008}.
If a semimeasure assigns zero probability to every finite string,
then it is called a \emph{measure}. 
Measures and semimeasures are uniquely defined by their values on
cylinder sets.
For convenience
we identify a string $x \in \X^*$ with its cylinder set $\Gamma_x$.

For two functions $f, g: \X^* \to \mathbb{R}$
we use the notation $f \timesgeq g$ iff
there is a constant $c > 0$ such that $f(x) \geq cg(x)$ for all $x \in \X^*$.
Moreover, we define $f \timesleq g$ iff $g \timesgeq f$ and
we define $f \timeseq g$ iff $f \timesleq g$ and $f \timesgeq g$.
Note that $f \timeseq g$ does \emph{not} imply that
there is a constant $c$ such that $f(x) = cg(x)$ for all $x$.

Let $U$ denote some universal Turing machine.
The \emph{Kolmogorov complexity $K(x)$} of a string $x$ is
the length of the shortest program on $U$ that prints $x$ and then halts.
A string $x$ is \emph{incompressible} iff $K(x) \geq |x|$.
We define $m(t) := \min_{n \geq t} K(n)$,
the \emph{monotone lower bound on $K$}.
Note that $m$ grows slower than any unbounded computable function.
(Its inverse is a version of the \emph{busy beaver} function.)
We also use the same machine $U$
as a monotone Turing machine by ignoring the halting state
and using a write-only output tape.
The \emph{monotone Kolmogorov complexity $\Km(x)$} denotes
the length of the shortest program on the monotone machine $U$
that prints a string starting with $x$.
Since monotone complexity does not require the machine to halt,
there is a constant $c$ such that  $\Km(x) \leq K(x) + c$ for all $x \in X^*$.

\emph{Solomonoff's prior $M$}~\cite{Solomonoff:1964} is defined as
the probability that the universal monotone Turing machine computes
a string when fed with fair coin flips in the input tape.
Formally,
\[
     M(x)
~:=~ \sum_{p:\, x \sqsubseteq U(p)} 2^{-|p|}.
\]
Equivalently, the Solomonoff prior $M$ can be defined as
a mixture over all lower semicomputable semimeasures~\cite{WSH:2011}.

The function $M$ is a lower semicomputable semimeasure,
but not computable and not a measure~\cite[Lem.\ 4.5.3]{LV:2008}.
It can be turned into a measure $M\norm$
using \emph{Solomonoff normalization}~\cite[Sec.\ 4.5.3]{LV:2008}:
$M\norm(\epsilon) := 1$ and
for all $x \in \X^*$ and $a \in \X$,
\begin{equation}\label{eq:normalization}
   M\norm(xa)
:= M\norm(x) \frac{M(xa)}{\sum_{b \in \X} M(xb)}
\end{equation}
since $M(x) > 0$ for all $x \in \X^*$.

Every program contributes to $M$, so
we have that $M(x) \geq 2^{-\Km(x)}$.
However, the upper bound $M(x) \timesleq 2^{-\Km(x)}$ is generally false%
~\cite{Gacs:1983}.
Instead, the following weaker statement holds.
\begin{lemma}[{\cite{Levin:1974} as cited in \cite[p.\ 75]{Gacs:1983}}]
\label{lem:Levin}
Let $E \subset \X^*$ be a recursively enumerable and prefix-free set.
Then there is a constant $c_E \in \mathbb{N}$ such that
$M(x) \leq 2^{-\Km(x)+c_E}$ for all $x \in E$.
\end{lemma}
\begin{proof}
Define
\[
\nu(x) :=
\begin{cases}
M(x), &\text{if } x \in E, \text{ and} \\
0, &\text{otherwise}.
\end{cases}
\]
The semimeasure $\nu$ is lower semicomputable
because $E$ is recursively enumerable.
Furthermore, $\sum_{x \in \X^*} \nu(x) \leq 1$
because $M$ is a semimeasure and $E$ is prefix-free.
Therefore $\nu$ is a discrete semimeasure.
Hence there are constant $c$ and $c'$ such that
$    \Km(x)
\leq K(x) + c
\leq -\log \nu(x) + c + c'
=    -\log M(x) + c + c'$~\cite[Cor.\ 4.3.1]{LV:2008}.
\qed
\end{proof}

\begin{lemma}[{\cite[Sec.\ 4.5.7]{LV:2008}}]
\label{lem:Martin-Loef}
For any computable measure $\mu$
the set of $\mu$-Martin-Löf-random sequences has $\mu$-probability one:
\[
  \mu(\{ x \in \X^\infty
        \mid \exists c \forall t.\; M(x_{1:t}) \leq c \mu(x_{1:t}) \})
= 1.
\]
\end{lemma}

\section{Solomonoff and the Black Ravens}
\label{sec:Solomonoff-and-the-black-ravens}

\paragraph{Setup.}
In order to formalize the black raven problem
(in line with \cite[Sec.\ 7.4]{RH:2011}),
we define two predicates: blackness $B$ and ravenness $R$.
There are four possible observations:
a black raven $\BR$,
a non-black raven $\nBR$,
a black non-raven $\BnR$, and
a non-black non-raven $\nBnR$.
Therefore our alphabet consists of
four symbols corresponding to each of the possible observations,
$\X := \{ \BR, \nBR, \BnR, \nBnR \}$.
We will not make the formal distinction between
observations and the symbols that represent them,
and simply use both interchangeably.

We are interested in the hypothesis `all ravens are black'.
Formally, it corresponds to the set
\begin{equation}\label{def:H}
     H
~:=~ \{ x \in \X^\sharp \mid x_t \neq \nBR \;\forall t \}
~ =~ \{ \BR, \BnR, \nBnR \}^\sharp,
\end{equation}
the set of all finite and infinite strings
in which the symbol $\nBR$ does not occur.
Let $H^c := \X^\sharp \setminus H$ be the complement hypothesis
`there is at least one non-black raven'.
We fix the definition of $H$ and $H^c$ for the rest of this paper.

Using Solomonoff induction,
our prior belief in the hypothesis $H$ is
\[
    M(H)
~=~ \sum_{p:\, U(p) \in H} 2^{-|p|},
\]
the cumulative weight of all programs that do not print any non-black ravens.
In each time step $t$,
we make one observation $x_t \in \X$.
Our \emph{history} $x_{<t} = x_1 x_2 \ldots x_{t-1}$
is the sequence of all previous observations.
We update our belief with Bayes' rule
in accordance with the Bayesian framework for learning~\cite{Jaynes:2003}:
our \emph{posterior belief} in the hypothesis $H$ is
\[
    M(H \mid x_{1:t})
~=~ \frac{M(H \cap x_{1:t})}{M(x_{1:t})}.
\]
We say that the observation $x_t$ \emph{confirms} the hypothesis $H$ iff
$M(H \mid x_{1:t}) > M(H \mid x_{<t})$ (the belief in $H$ increases), and
we say that the observation $x_t$ \emph{disconfirms} the hypothesis $H$ iff
$M(H \mid x_{1:t}) < M(H \mid x_{<t})$ (the belief in $H$ decreases).
If $M(H \mid x_{1:t}) = 0$, we say that $H$ is \emph{refuted}, and
if $M(H \mid x_{1:t}) \to 1$ as $t \to \infty$,
we say that $H$ is \emph{(maximally) confirmed asymptotically}.

\paragraph{Confirmation and Refutation.}
Let the sequence $x_{1:\infty}$ be sampled
from a computable measure $\mu$, the \emph{true environment}.
If we observe a non-black raven, $x_t = \nBR$,
the hypothesis $H$ is refuted
since $H \cap x_{1:t} = \emptyset$
and this implies $M(H \mid x_{1:t}) = 0$.
In this case, our enquiry regarding $H$ is settled.
For the rest of this paper, we focus on the interesting case:
we assume our hypothesis $H$ is in fact true in $\mu$ ($\mu(H) = 1$),
i.e., $\mu$ does not generate any non-black ravens.
Since Solomonoff's prior $M$ dominates all computable measures,
there is a constant $w_\mu$ such that
\begin{equation}\label{eq:universal-dominance}
\forall x \in \X^* \quad M(x) \geq w_\mu \mu(x).
\end{equation}
Thus
Blackwell and Dubins' famous merging of opinions theorem~\cite{BD:1962}
implies
\begin{equation}\label{eq:Blackwell-Dubins}
M(H \mid x_{1:t}) \to 1
\text{ as $t \to \infty$ with $\mu$-probability one}.\footnote{%
Blackwell-Dubins' theorem refers to (probability) measures,
but technically $M$ is a semimeasure.
However, we can view $M$ as a measure
by introducing an extra symbol to our alphabet~\cite[p.\ 264]{LV:2008}.
This preserves dominance \eqref{eq:universal-dominance},
and hence absolute continuity,
which is the precondition for Blackwell-Dubins' theorem.
}
\end{equation}
Therefore our hypothesis $H$ is confirmed asymptotically~\cite[Sec.\ 7.4]{RH:2011}.
However, convergence to $1$ is extremely slow,
slower than any unbounded computable function,
since $1 - M(H \mid x_{1:t}) \timesgeq 2^{-m(t)}$ for all $t$.

In our setup,
the equivalence condition holds trivially:
a logically equivalent way of formulating a hypothesis
yields the same set of infinite strings,
therefore in our formalization it constitutes the same hypothesis.
%
The central question of this paper is Nicod's criterion,
which refers to the assertion that
$\BR$ and $\nBnR$ confirm $H$, i.e.,
$M(H \mid x_{1:t} \BR) > M(H \mid x_{<t})$ and
$M(H \mid x_{1:t} \nBnR) > M(H \mid x_{<t})$ for all strings $x_{<t}$.

\section{Disconfirming H}
\label{sec:disconfirming-H}

We first illustrate the violation of Nicod's criterion
by defining a particular universal Turing machine.

\begin{example}[Black Raven Disconfirms]\label{ex:black-raven-disconfirms}
The observation of a black raven can falsify a short program
that supported the hypothesis $H$.
Let $\varepsilon > 0$ be a small rational number.
We define a semimeasure $\rho$ as follows.
\begin{align*}
\rho(\nBnR^\infty) &:= \tfrac{1}{2}
&
\rho(\BR^\infty)      &:= \tfrac{1}{4}
&
\rho(\BR\, \nBR^\infty) &:= \tfrac{1}{4} - \varepsilon
&
\rho(x) &:= 0 \text{ otherwise}.
\end{align*}
To get a universally dominant semimeasure $\xi$,
we mix $\rho$ with the universally dominant semimeasure $M$.
\[
\xi(x) := \rho(x) + \varepsilon M(x).
\]
For computable $\varepsilon$, the mixture $\xi$ is
a lower semicomputable semimeasure.
Hence there is a universal monotone Turing machine
whose Solomonoff prior is equal to $\xi$~\cite[Lem.\ 13]{WSH:2011}.
Our a priori belief in $H$ at time $t = 0$ is
\[
     \xi(H \mid \epsilon)
=    \xi(H)
\geq \rho(\nBnR^\infty) + \rho(\BR^\infty)
=    75\%,
\]
while our a posteriori belief in $H$ after seeing a black raven is
\[
     \xi(H \mid \BR)
=    \frac{\xi(H \cap \BR)}{\xi(\BR)}
\leq \frac{\rho(\BR^\infty) + \varepsilon}{\rho(BR^\infty) + \rho(\BR\nBR^\infty)}
=    \frac{\tfrac{1}{4} + \varepsilon}{\tfrac{1}{2} - \varepsilon}
<    75 \%
\]
for $\varepsilon \leq 7\%$.
Hence observing a black raven in the first time step disconfirms
the hypothesis $H$.
\hfill$\Diamond$ 
\end{example}

The rest of this section is dedicated to show that
this effect occurs independent of the universal Turing machine $U$
and on all computable infinite strings.

\begin{figure}[t]
\begin{minipage}{0.495\textwidth}
\centering
\begingroup
\setlength{\tabcolsep}{10pt}
\renewcommand{\arraystretch}{2}
\begin{tabular}{r|cc}
$M(\,\cdot\,)$                          & $H$ & $H^c$ \\
\hline
$\bigcup_{a \neq x_t} \Gamma_{x_{<t}a}$ & $A$ & $B$ \\
$\Gamma_{x_{1:t}}$                      & $C$ & $D$ \\
$\{ x_{<t} \}$                          & $E$ & $0$
\end{tabular}

\endgroup
\end{minipage}
\begin{minipage}{0.495\textwidth}
\begin{align*}
A &:= \sum_{a \neq x_t} M(x_{<t}a \cap H) \\
B &:= \sum_{a \neq x_t} M(x_{<t}a \cap H^c) \\
C &:= M(x_{1:t} \cap H) \\
D &:= M(x_{1:t} \cap H^c) \\
E &:= M(x_{<t}) - \sum_{a \in \X} M(x_{<t}a)
\end{align*}
\end{minipage}
\caption{
The definitions of the values $A$, $B$, $C$, $D$, and $E$.
Note that by assumption,
$x_{<t}$ does not contain non-black ravens,
therefore $M(\{ x_{<t} \} \cap H^c) = M(\emptyset) = 0$.
}
\label{fig:ABCDE}
\end{figure}

\subsection{Setup}
\label{ssec:setup}

At time step $t$, we have seen the history $x_{<t}$
and now update our belief using the new symbol $x_t$.
To understand what happens,
we split all possible programs into five categories.
\begin{enumerate}[(a)]
\item Programs that \emph{never} print non-black ravens (compatible with $H$),
	but become falsified at time step $t$
	because they print a symbol other than $x_t$.
\item Programs that eventually print a non-black raven (contradict $H$),
	but become falsified at time step $t$
	because they print a symbol other than $x_t$.
\item Programs that \emph{never} print non-black ravens (compatible with $H$),
	and predict $x_t$ correctly.
\item Programs that eventually print a non-black raven (contradict $H$),
	and predict $x_t$ correctly.
\item Programs that do not print additional symbols after printing $x_{<t}$
	(because they go into an infinite loop).
\end{enumerate}
Let $A$, $B$, $C$, $D$, and $E$ denote the cumulative contributions of
these five categories of programs to $M$.
A formal definition is given in \autoref{fig:ABCDE},
and implicitly depends on
the current time step $t$ and the observed string $x_{1:t}$.
The values of $A$, $B$, $C$, $D$, and $E$ are in the interval $[0, 1]$
since they are probabilities.
Moreover, the following holds.
\begin{align}
   M(x_{<t})
&= A + B + C + D + E
&
   M(x_{1:t})
&= C + D
\label{eq:ABCDE-x} \\
   M(x_{<t} \cap H)
&= A + C + E
&
   M(x_{1:t} \cap H)
&= C
\label{eq:ABCDE-H} \\
   M(H \mid x_{<t})
&= \frac{A + C + E}{A + B + C + D + E}
&
   M(H \mid x_{1:t})
&= \frac{C}{C + D}
\label{eq:ABCDE-conditional}
\end{align}

We use results from algorithmic information theory to
derive bounds on $A$, $B$, $C$, $D$, and $E$.
This lets us apply
the following lemma which states
a necessary and sufficient condition
for confirmation/disconfirmation at time step $t$.

\begin{lemma}[Confirmation Criterion]
\label{lem:decrease-ABCDE}
Observing $x_t$ confirms (disconfirms) the hypothesis $H$ if and only if
$AD + DE < BC$ ($AD + DE > BC$).
\end{lemma}
\begin{proof}
The hypothesis $H$ is confirmed if and only if
\begin{align*}
  M(H \mid x_{1:t}) - M(H \mid x_{<t})
\stackrel{\eqref{eq:ABCDE-conditional}}{=}
  \tfrac{C}{C + D} - \tfrac{A + C + E}{A + B + C + D + E}
= \tfrac{BC - AD - DE}{(A + B + C + D + E)(C + D)}
\end{align*}
is positive.
Since the denominator is positive,
this is equivalent to $BC > AD + DE$.
\qed
\end{proof}

\begin{example}[Confirmation Criterion Applied to \autoref{ex:black-raven-disconfirms}]
\label{ex:belief-decreases2}
In \autoref{ex:black-raven-disconfirms} we picked a particular universal prior
and $x_1 = \BR$.
In this case,
the values for $A$, $B$, $C$, $D$, and $E$ are
\begin{align*}
A &\in [\tfrac{1}{2}, \tfrac{1}{2} + \varepsilon]
&
B &\in [0, \varepsilon]
&
C &\in [\tfrac{1}{4}, \tfrac{1}{4} + \varepsilon]
&
D &\in [\tfrac{1}{4} - \varepsilon, \tfrac{1}{4}]
&
E &\in [0, \varepsilon].
\end{align*}
We invoke \autoref{lem:decrease-ABCDE}
with $\varepsilon := 7\%$
to get that $x_1 = \BR$ disconfirms $H$:
\[
     AD + DE
\geq \tfrac{1}{8} - \tfrac{\varepsilon}{2}
=    0.09
>    0.0224
=    \tfrac{\varepsilon}{4} + \varepsilon^2
\geq BC.
\eqno\Diamond
\]
\end{example}

\begin{lemma}[Bounds on $ABCDE$]\label{lem:bounds-ABCDE}
Let $x_{1:\infty} \in H$ be some computable infinite string.
The following statements hold for every time step $t$.
\begin{multicols}{2}
\begin{enumerate}[(i)]
\item $0 < A, B, C, D, E < 1$
	\label{itm:0<ABCDE<1}
\item $A + B \timesleq 2^{-K(t)}$
	\label{itm:AB<=}
\item $A, B \timesgeq 2^{-K(t)}$
	\label{itm:AB>=}
\item $C \timesgeq 1$
	\label{itm:C>=}
\item $D \timesgeq 2^{-m(t)}$
	\label{itm:D>=}
\item $D \to 0$ as $t \to \infty$
	\label{itm:D->0}
\item $E \to 0$ as $t \to \infty$
	\label{itm:E->0}
\end{enumerate}
\end{multicols}
\end{lemma}
\begin{proof}
Let $p$ be a program that computes the infinite string $x_{1:\infty}$.
\begin{enumerate}[(i)]
\item Each of $A, B, C, D, E$ is a probability value and
	hence bounded between $0$ and $1$.
	These bounds are strict because for any finite string
	there is a program that prints that string.
\item A proof is given in the appendix of \cite{Hutter:2007universal}.
	Let $a \neq x_t$ and
	let $q$ be the shortest program for the string $x_{<t}a$,
	i.e., $|q| = \Km(x_{<t}a)$.
	We can reconstruct $t$ by running $p$ and $q$ in parallel
	and counting the number of characters printed until their output differs.
	Therefore there is a constant $c$ independent of $t$ such that
	$K(t) \leq |p| + |q| + c = |p| + \Km(x_{<t}a) + c$.
	Hence
	\begin{equation}\label{eq:Kmt}
	     2^{-\Km(x_{<t}a)}
	\leq 2^{-K(t) + |p| + c}
	\end{equation}
	The set $E := \{ x_{<t}a \mid t \in \mathbb{N}, a \neq x_t \}$
	is recursively enumerable and prefix-free,
	so \autoref{lem:Levin} yields a constant $c_E$ such that
	\[
	     M(x_{<t} a)
	\leq 2^{-\Km(x_{<t} a) + c_E}
	\stackrel{\eqref{eq:Kmt}}{\leq}
	     2^{-K(t) + |p| + c + c_E}.
	\]
	With $A + B \leq (\#\X - 1) \max_{a \neq x_t} M(x_{<t}a)$ follows the claim.
\item Let $a \neq x_t$ and
	let $q$ be the shortest program to compute $t$,
	i.e., $|q| = K(t)$.
	We can construct a program that prints $x_{<t}a\nBR$
	by first running $q$ to get $t$ and then running $p$
	until it has produced a string of length $t - 1$,
	and then printing $a\nBR$.
	Hence there is a constant $c$ independent of $t$ such that
	$\Km(x_{<t}a\nBR) \leq |q| + |p| + c = K(t) + |p| + c$.
	Therefore
	\[
	     M(x_{<t}a \cap H^c)
	\geq M(x_{<t}a\nBR)
	\geq 2^{-\Km(x_{<t}a\nBR)}
	\geq 2^{-K(t) - |p| - c}.
	\]
	For the bound on $M(x_{<t}a \cap H)$
	we proceed analogously except that
	instead of printing $\nBR$ the program goes into an infinite loop.
\item Since by assumption the program $p$ computes $x_{1:\infty} \in H$,
	we have that $M(x_{1:t} \cap H) \geq 2^{-|p|}$.
\item Let $n$ be an integer such that $K(n) = m(t)$.
	We proceed analogously to (\ref{itm:AB>=})
	with a program $q$ that prints $n$ such that $|q| = m(t)$.
	Next, we write a program that produces the output $x_{1:n} \nBR$,
	which yields a constant $c$ independent of $t$ such that
	\[
	     M(x_{1:t} \cap H^c)
	\geq M(x_{1:n}\nBR)
	\geq 2^{-\Km(x_{1:n}\nBR)}
	\geq 2^{-|q| - |p| - c}
	=    2^{-m(t) - |p| - c}.
	\]
\item This follows from Blackwell and Dubins' result \eqref{eq:Blackwell-Dubins}:
	\[
		 D
	=    (C + D) \left( 1 - \tfrac{C}{C + D} \right)
	\leq (1 + 1) (1 - M(H \mid x_{1:t}))
	\to  0 \text{ as } t \to \infty.
	\]
\item $\sum_{t=1}^\infty M(\{ x_{<t} \})
	= M(\{ x_{<t} \mid t \in \mathbb{N} \}) \leq 1$,
	thus $E = M(\{ x_{<t} \}) \to 0$.
	\qed
\end{enumerate}
\end{proof}

\autoref{lem:bounds-ABCDE} states the bounds that illustrate
the ideas to our results informally:
From $A \timeseq B \timeseq 2^{-K(t)}$
(\ref{itm:AB<=},\ref{itm:AB>=}) and $C \timeseq 1$ (\ref{itm:C>=})
we get
\begin{align*}
AD &\timeseq 2^{-K(t)}D,
&
BC &\timeseq 2^{-K(t)}.
\end{align*}
According to \autoref{lem:decrease-ABCDE},
the sign of $AD + DE - BC$ tells us
whether our belief in $H$ increases (negative) or decreases (positive).

Since $D \to 0$ (\ref{itm:D->0}),
the term $AD \timeseq 2^{-K(t)}D$ will eventually be smaller than
$BC \timeseq 2^{-K(t)}$.
Therefore it is crucial how fast $E \to 0$ (\ref{itm:E->0}).
If we use $M$, then $E \to 0$ slower than $D \to 0$ (\ref{itm:D>=}),
therefore $AD + DE - BC$ is positive infinitely often
(\autoref{thm:decrease-io}).
If we use $M\norm$ instead of $M$, then $E = 0$ and hence
$AD + DE - BC = AD - BC$ is negative except for a finite number of steps
(\autoref{thm:Mnorm-on-sequence}).

\subsection{Unnormalized Solomonoff Prior}
\label{ssec:unnormalized-Solomonoff-prior}

\begin{theorem}[Counterfactual Black Raven Disconfirms H]
\label{thm:M-decreases}
Let $x_{1:\infty}$ be a computable infinite string such that
$x_{1:\infty} \in H$ ($x_{1:\infty}$ does not contain any non-black ravens)
and $x_t \neq \BR$ infinitely often.
Then there is a time step $t \in \mathbb{N}$ (with $x_t \neq \BR$) such that
$
  M(H \mid x_{<t} \BR)
< M(H \mid x_{<t})
$.
\end{theorem}
\begin{proof}
Let $t$ be time step such that $x_t \neq \BR$.
From the proof of \autoref{lem:bounds-ABCDE} (\ref{itm:AB>=})
we get $M(H^c \cap x_{<t}\BR) \geq 2^{-K(t)-c}$ and thus
\begin{align*}
     M(H \mid x_{<t}\BR)
&\leq \frac{M(H \cap x_{<t}\BR) + M(H^c \cap x_{<t}\BR) - 2^{-K(t)-c}}{M(x_{<t} \BR)} \\
&=    1 - \frac{2^{-K(t)-c}}{M(x_{<t} \BR)}
\leq 1 - \frac{2^{-K(t)-c}}{A + B}
\stackrel{(\ref{itm:AB<=})}{\leq}
     1 - 2^{-c-c'}.
\end{align*}
From \eqref{eq:Blackwell-Dubins} there is a $t_0$
such that for all $t \geq t_0$ we have
$M(H \mid x_{<t}) > 1 - 2^{-c-c'} \geq M(H \mid x_{<t}\BR)$.
Since $x_t \neq \BR$ infinitely often according to the assumption,
there is a $x_t \neq \BR$ for $t \geq t_0$.
\qed
\end{proof}

Note that
the black raven in \autoref{thm:M-decreases} that we observe at time $t$
is \emph{counterfactual}, i.e.,
not part of the sequence $x_{1:\infty}$.
If we picked the binary alphabet $\{ \BR, \nBR \}$
and denoted only observations of ravens,
then \autoref{thm:M-decreases} would not apply:
the only infinite string in $H$ is $\BR^\infty$ and
the only counterfactual observation is $\nBR$,
which immediately falsifies the hypothesis $H$.
The following theorem gives an on-sequence result.

\begin{theorem}[Disconfirmation Infinitely Often for $M$]
\label{thm:decrease-io}
Let $x_{1:\infty}$ be a computable infinite string such that
$x_{1:\infty} \in H$ ($x_{1:\infty}$ does not contain any non-black ravens).
Then $M(H \mid x_{1:t}) < M(H \mid x_{<t})$
for infinitely many time steps $t \in \mathbb{N}$.
\end{theorem}
\begin{proof}
We show that there are infinitely many $n \in \mathbb{N}$ such that
for each $n$ there is a time step $t > n$ where the belief in $H$ decreases.
The $n$s are picked to have low Kolmogorov complexity,
while the $t$s are incompressible.
The crucial insight is that
a program that goes into an infinite loop at time $t$
only needs to know $n$ and not $t$,
thus making this program much smaller than $K(t) \geq \log t$.

Let $q_n$ be a program that
starting with $t = n + 1$ incrementally outputs $x_{1:t}$
as long as $K(t) < \log t$.
Formally, let $\phi(y, k)$ be a computable function such that
$\phi(y, k + 1) \leq \phi(y, k)$ and $\lim_{k \to \infty} \phi(y, k) = K(y)$.
\begin{center}
\begin{minipage}{57mm}
\begin{lstlisting}
program $q_n$:
    $t$ := $n + 1$
    output $x_{<t}$
    while true:
        $k$ := $0$
        while $\phi(t, k) \geq \log t$:
            $k$ := $k + 1$
        output $x_t$
        $t$ := $t + 1$
\end{lstlisting}
\end{minipage}
\end{center}
The program $q_n$ only needs to know $p$ and $n$,
so we have that $|q_n| \leq K(n) + c$
for some constant $c$ independent of $n$ and $t$.
For the smallest $t > n$ with $K(t) \geq \log t$,
the program $q_n$ will go into an infinite loop
and thus fail to print a $t$-th character.
Therefore
\begin{equation}\label{eq:E}
E = M(\{ x_{<t} \}) \geq 2^{-|q_n|} \geq 2^{-K(n)-c}.
\end{equation}

Incompressible numbers are very dense,
and a simple counting argument shows that there must be one
between $n$ and $4n$~\cite[Thm.\ 3.3.1 (i)]{LV:2008}. 
Furthermore, we can assume that $n$ is large enough such that
$m(4n) \leq m(n) + 1$ (since $m$ grows slower than the logarithm).
Then
\begin{equation}\label{eq:m-and-K}
m(t) \leq m(4n) \leq m(n) + 1 \leq K(n) + 1.
\end{equation}
Since the function $m$ grows slower than any unbounded computable function,
we find infinitely many $n$ such that
\begin{equation}\label{eq:K-bound}
K(n) \leq \tfrac{1}{2} (\log n - c - c' - c'' - 1),
\end{equation}
where $c'$ and $c''$ are the constants from
\autoref{lem:bounds-ABCDE} (\ref{itm:AB<=},\ref{itm:D>=}).
For each such $n$,
there is a $t > n$ with $K(t) \geq \log t$, as discussed above.
This entails
\begin{equation}\label{eq:n-and-t}
     m(t) + K(n) + c + c''
\stackrel{\eqref{eq:m-and-K}}{\leq}
     2K(n) + 1 + c + c''
\stackrel{\eqref{eq:K-bound}}{\leq}
     \log n - c'
\leq \log t - c'
\leq K(t) - c'.
\end{equation}
From \autoref{lem:bounds-ABCDE} we get
\[
     AD + DE
\stackrel{(\ref{itm:0<ABCDE<1})}{>}
     DE
\stackrel{\eqref{eq:E},(\ref{itm:D>=})}{\geq}
     2^{-m(t) - c - K(n) - c''}
\stackrel{\eqref{eq:n-and-t}}{\geq}
     2^{-K(t) + c'}
\stackrel{(\ref{itm:0<ABCDE<1},\ref{itm:AB<=})}{\geq}
     BC.
\]
With \autoref{lem:decrease-ABCDE} we conclude that
$x_t$ disconfirms $H$.
\qed
\end{proof}

To get that $M$ violates Nicod's criterion infinitely often,
we apply \autoref{thm:decrease-io}
to the computable infinite string $\BR^\infty$.

\subsection{Normalized Solomonoff Prior}
\label{ssec:normalized-Solomonoff-prior}

In this section we show that for computable infinite strings,
our belief in the hypothesis $H$ is non-increasing at most finitely many times
if we normalize $M$.

For this section we define $A'$, $B'$, $C'$, $D'$, and $E'$
analogous to $A$, $B$, $C$, $D$, and $E$
as given in \autoref{fig:ABCDE} with $M\norm$ instead of $M$.

\begin{lemma}[$M\norm \geq M$]
\label{lem:Mnorm-dominates-M}
$M\norm(x) \geq M(x)$ for all $x \in \X^*$.
\end{lemma}
\begin{proof}
We use induction on the length of $x$:
$M\norm(\epsilon) = 1 = M(\epsilon)$ and
\[
     M\norm(xa)
=    \frac{M\norm(x) M(xa)}{\sum_{b \in \X} M(xb)}
\geq \frac{M(x) M(xa)}{\sum_{b \in \X} M(xb)}
\geq \frac{M(x) M(xa)}{M(x)}
=    M(xa).
\]
The first inequality holds by induction hypothesis and
the second inequality uses the fact that $M$ is a semimeasure.
\qed
\end{proof}
The following lemma states the same bounds for $M\norm$
as given in \autoref{lem:bounds-ABCDE}
except for (\ref{itm:0<ABCDE<1}) and (\ref{itm:E->0}).

\begin{lemma}[Bounds on $A'B'C'D'E'$]
\label{lem:bounds-ABCDE'}
Let $x_{1:\infty} \in H$ be some infinite string computed by program $p$.
The following statements hold for all time steps $t$.
\begin{multicols}{2}
\begin{enumerate}[(i)]
\item \label{itm:ABCDE<=A'B'C'D'E'}
	$A \leq A'$, $B \leq B'$, \\ $C \leq C'$, $D \leq D'$
\item \label{itm:A'B'<=}
	$A' + B' \timesleq 2^{-K(t)}$
\item \label{itm:A'B'>=}
	$A', B' \timesgeq 2^{-K(t)}$
\item \label{itm:C'>=}
	$C' \timesgeq 1$
\item \label{itm:D'>=}
	$D' \timesgeq 2^{-m(t)}$
\item \label{itm:D'->0}
	$D' \to 0$ as $t \to \infty$
\item \label{itm:E'=0}
	$E' = 0$
\end{enumerate}
\end{multicols}
\end{lemma}
\begin{proof}
\begin{enumerate}[(i)]
\item[(i)] Follows from \autoref{lem:Mnorm-dominates-M}.
\item[(ii)] Let $a \neq x_t$.
	From \autoref{lem:bounds-ABCDE} (\ref{itm:AB<=}) we have
	$M(x_{<t}a) \timesleq 2^{-K(t)}$.
	Thus
	\[
	          M\norm(x_{<t}a)
	\stackrel{\eqref{eq:normalization}}{=}
	          \frac{M\norm(x_{<t}) M(x_{<t}a)}{\sum_{b \in \X} M(x_{<t}b)}
	\timesleq \frac{M\norm(x_{<t}) 2^{-K(t)}}{\sum_{b \in \X} M(x_{<t}b)}
	\timesleq 2^{-K(t)}.
	\]
	The last inequality follows from
	$\sum_{b \in \X} M(x_{<t} b) \geq M(x_{1:t}) \timesgeq 1$
	(\autoref{lem:bounds-ABCDE} (\ref{itm:C>=})) and
	$M\norm(x_{<t}) \leq 1$.
\item[(iii-v)] This is a consequence of (\ref{itm:ABCDE<=A'B'C'D'E'})
and \autoref{lem:bounds-ABCDE} (\ref{itm:AB>=}-\ref{itm:D>=}).
\item[(vi)] Blackwell and Dubins' result also applies to $M\norm$,
	therefore the proof of \autoref{lem:bounds-ABCDE} (\ref{itm:D->0})
	goes through unchanged.
\item[(vii)] Since $M\norm$ is a measure,
	it assigns zero probability to finite strings,
	i.e., $M\norm(\{ x_{<t} \}) = 0$, hence $E' = 0$.
	\qed
\end{enumerate}
\end{proof}

\begin{theorem}[Disconfirmation Finitely Often for $M\norm$]
\label{thm:Mnorm-on-sequence}
Let $x_{1:\infty}$ be a computable infinite string such that
$x_{1:\infty} \in H$ ($x_{1:\infty}$ does not contain any non-black ravens).
Then there is a time step $t_0$ such that
$
  M\norm(H \mid x_{1:t})
> M\norm(H \mid x_{<t})
$ for all $t \geq t_0$.
\end{theorem}

Intuitively, at time step $t_0$,
$M\norm$ has learned that it is observing the infinite string $x_{1:\infty}$
and there are no short programs remaining that support the hypothesis $H$
but predict something other than $x_{1:\infty}$.

\begin{proof}
We use \autoref{lem:bounds-ABCDE'} (\ref{itm:A'B'<=},\ref{itm:A'B'>=},\ref{itm:C'>=},\ref{itm:E'=0}) to conclude
\[
     A'D' + D'E' - B'C'
\leq 2^{-K(t)+c} D' + 0 - 2^{-K(t)-c'-c''}
\leq 2^{-K(t)+c} (D' - 2^{-c-c'-c''}).
\]
From \autoref{lem:bounds-ABCDE'} (\ref{itm:D'->0}) we have that $D' \to 0$,
so there is a $t_0$ such that for all $t \geq t_0$
we have $D' < 2^{-c-c'-c''}$.
Thus $A'D' + D'E' - B'C'$ is negative for $t \geq t_0$.
Now \autoref{lem:decrease-ABCDE}
entails that the belief in $H$ increases.
\qed
\end{proof}

Interestingly, \autoref{thm:Mnorm-on-sequence} does not hold for $M$
since that would contradict \autoref{thm:decrease-io}.
The reason is that there are quite short programs that produce $x_{<t}$,
but do not halt after that.
However, from $p$ and $x_{<t}$ we cannot reconstruct $t$,
hence a program for $x_{<t}$ does not give us a bound on $K(t)$.

Since we get the same bounds for $M\norm$ as in \autoref{lem:bounds-ABCDE},
the result of \autoref{thm:M-decreases} transfers to $M\norm$:

\begin{corollary}[Counterfactual Black Raven Disconfirms $H$]
\label{cor:Mnorm-decreases}
Let $x_{1:\infty}$ be a computable infinite string such that
$x_{1:\infty} \in H$ ($x_{1:\infty}$ does not contain any non-black ravens)
and $x_t \neq \BR$ infinitely often.
Then there is a time step $t \in \mathbb{N}$ (with $x_t \neq \BR$) such that
$
  M\norm(H \mid x_{<t} \BR)
< M\norm(H \mid x_{<t})
$.
\end{corollary}

For incomputable infinite strings
the belief in $H$ can decrease infinitely often:

\begin{corollary}[Disconfirmation Infinitely Often for $M\norm$]
\label{cor:Mnorm-decreases-io}
There is an (incomputable) infinite string $x_{1:\infty} \in H$ such that
$M\norm(H \mid x_{1:t}) < M\norm(H \mid x_{<t})$
infinitely often as $t \to \infty$.
\end{corollary}
\begin{proof}
We iterate \autoref{cor:Mnorm-decreases}:
starting with $\nBnR^\infty$, we get a time step $t_1$ such that
observing $\BR$ at time $t_1$ disconfirms $H$.
We set $x_{1:t_1} := \nBnR^{t_1-1}\BR$ and apply \autoref{cor:Mnorm-decreases} to
$x_{1:t_1} \nBnR^\infty$ to get a time step $t_2$ such that
observing $\BR$ at time $t_2$ disconfirms $H$.
Then we set $x_{1:t_2} := x_{1:t_1} \nBnR^{t_2 - t_1 - 1} \BR$, and so on.
\qed
\end{proof}

\subsection{Stochastically Sampled Strings}


The proof techniques from the previous subsections do not generalize to
strings that are sampled stochastically.
The main obstacle is the complexity of counterfactual observations
$x_{<t}a$ with $a \neq x_t$:
for deterministic strings $\Km(x_{<t}a) \to 0$,
while for stochastically sampled strings $\Km(x_{<t}a) \nrightarrow 0$.
Consider the following example.

\begin{example}[Uniform IID Observations]\label{ex:uniform}
Let $\lambda_H$ be a measure that
generates uniform i.i.d.\ symbols from $\{ \BR, \BnR, \nBnR \}$.
Formally,
\[
\lambda_H(x) :=
\begin{cases}
0        &\text{if } \nBR \in x \text{, and} \\
3^{-|x|} &\text{otherwise}.
\end{cases}
\]
By construction, $\lambda_H(H) = 1$.
By \autoref{lem:Martin-Loef} we have
$A,C,E \timeseq 3^{-t}$ and $B, D \timeseq 3^{-t} 2^{-m(t)}$
with $\lambda_H$-probability one.
According to \autoref{lem:decrease-ABCDE},
the sign of $AD + DE - BC$ is indicative for the change in belief in $H$.
But this is inconclusive both for $M$ and $M\norm$
since each of the summands $AD$, $BC$, and $DE$ (in case $E \neq 0$)
go to zero at the same rate:
\[
         AD
\timeseq DE
\timeseq BC
\timeseq 3^{-2t} 2^{-m(t)}.
\]
Whether $H$ gets confirmed or disconfirmed
thus depends on the universal Turing machine and/or
the probabilistic outcome of the string drawn from $\lambda_H$.
\hfill$\Diamond$
\end{example}

\section{Discussion}

We chose to present our results in the setting of the black raven problem
to make them more accessible to intuition and
more relatable to existing literature.
But these results hold more generally:
our proofs follow from the bounds on $A$, $B$, $C$, $D$, and $E$
given in \autoref{lem:bounds-ABCDE} and \autoref{lem:bounds-ABCDE'}.
These bounds rely on the fact that we are observing a computable infinite string
and that at any time step $t$
there are programs consistent with the observation history
that contradict the hypothesis and
there are programs consistent with the observation history
that are compatible with the hypothesis.
No further assumptions on the alphabet, the hypothesis $H$, or
the universal Turing machine are necessary.

In our formalization of the raven problem
given in \autoref{sec:Solomonoff-and-the-black-ravens},
we used an alphabet with four symbols.
Each symbol indicates one of four possible types of observations
according to the two binary predicates blackness and ravenness.
One could object that
this formalization discards important structure from the problem:
$\BR$ and $\nBR$ have more in common than $\BR$ and $\nBnR$,
yet as symbols they are all the same.
Instead, we could use the latin alphabet and
spell out `black', `non-black', `raven', and `non-raven'.
The results given in this paper would still apply analogously.

Our result that Solomonoff induction does not satisfy Nicod's criterion
is not true for every time step, only for some of them.
Generally,
whether Nicod's criterion should be adhered to depends on
whether the paradoxical conclusion is acceptable.
A different Bayesian reasoner might be tempted to argue that a green apple
\emph{does} confirm the hypothesis $H$, but only to a small degree,
since there are vastly more non-black objects than ravens~\cite{Good:1960}.
This leads to the acceptance of the paradoxical conclusion,
and this solution to the confirmation paradox is known as
the \emph{standard Bayesian solution}.
It is equivalent to
the assertion that blackness is equally probable
regardless of whether $H$ holds:
$P(\text{black} | H) \approx P(\text{black})$%
~\cite{Vranas:2004}.
Whether or not this holds depends on our prior beliefs.

The following is a very concise example
against the standard Bayesian solution~\cite{Good:1967}:
There are two possible worlds,
the first has 100 black ravens and a million other birds,
while the second has 1000 black ravens, one white raven, and
a million other birds.
Now we draw a bird uniformly at random, and it turns out to be a black raven.
Contrary to what Nicod's criterion claims,
this is strong evidence that we are in fact in the second world,
and in this world non-black ravens exist.

For another, more intuitive example:
Suppose you do not know anything about ravens and
you have a friend who collects atypical objects.
If you see a black raven in her collection,
surely this would not
increase your belief in the hypothesis that all ravens are black.

We must conclude that violating Nicod's criterion is not
a fault of Solomonoff induction.
Instead, we should accept that for Bayesian reasoning
Nicod's criterion, in its generality, is false!
Quoting the great Bayesian master
E.\ T.\ Jay\-nes~\cite[p.\ 144]{Jaynes:2003}: 
\begin{quote}
In the literature there are perhaps 100 `paradoxes' and controversies
which are like this,
in that they arise from faulty intuition rather than faulty mathematics.
Someone asserts a general principle that seems to him intuitively right.
Then, when probability analysis reveals the error,
instead of taking this opportunity to educate his intuition,
he reacts by rejecting the probability analysis.
\end{quote}

\paragraph{Acknowledgement.}
This work was supported by ARC grant DP150104590.


\bibliographystyle{abbrv}
\bibliography{../ai,references}

\newpage
\section*{List of Notation}
\label{app:notation}

\begin{longtable}{lp{0.84\textwidth}}
$:=$
	& defined to be equal \\
$\#A$
	& the cardinality of the set $A$, i.e., the number of elements \\
$\X$
	& a finite alphabet \\
$\X^*$
	& the set of all finite strings over the alphabet $\X$ \\
$\X^\infty$
	& the set of all infinite strings over the alphabet $\X$ \\
$\X^\sharp$
	& $\X^\sharp := \X^* \cup \X^\infty$,
	the set of all finite and infinite strings over the alphabet $\X$ \\
$\Gamma_x$
	& the set of all finite and infinite strings that start with $x$ \\
$x, y$
	& finite or infinite strings, $x, y \in \X^\sharp$ \\
$x \sqsubseteq y$
	& the string $x$ is a prefix of the string $y$ \\
$\epsilon$
	& the empty string \\
$\varepsilon$
	& a small positive rational number \\
$t$
	& (current) time step \\
$n$
	& natural number \\
$K(x)$
	& Kolmogorov complexity of the string $x$:
	the length of the shortest program that prints $x$ and halts \\
$m(t)$
	& the monotone lower bound on $K$, formally $m(t) := \min_{n \geq t} K(n)$ \\
$\Km(x)$
	& monotone Kolmogorov complexity of the string $x$:
	the length of the shortest program on the monotone universal Turing machine
	that prints something starting with $x$ \\
$\BR$
	& a symbol corresponding to the observation of a black raven \\
$\nBR$
	& a symbol corresponding to the observation of a non-black raven \\
$\BnR$
	& a symbol corresponding to the observation of a black non-raven \\
$\nBnR$
	& a symbol corresponding to the observation of a non-black non-raven \\
$H$
	& the hypothesis `all ravens are black',
	formally defined in \eqref{def:H} \\
$U$
	& the universal (monotone) Turing machine \\
$M$
	& the Solomonoff prior \\
$M\norm$
	& the normalized Solomonoff prior,
	defined according to \eqref{eq:normalization} \\
$p, q$
	& programs on the universal (monotone) Turing machine
\end{longtable}

\end{document}